\documentclass[10pt,twocolumn,letterpaper]{article}

\usepackage[T1]{fontenc}
\usepackage[margin=0.75in,columnsep=0.25in]{geometry}
\usepackage[hyphens]{url}
\usepackage{graphicx}
\usepackage{natbib}
\usepackage{caption}
\usepackage{amsmath,amssymb,amsfonts,amsthm}
\usepackage{booktabs}
\usepackage{multirow}
\usepackage{placeins}
\usepackage[colorlinks=true,linkcolor=blue,citecolor=blue,urlcolor=blue]{hyperref}
\hypersetup{pdftitle={Neither Precision Nor Architecture Alone: Controlled Tests of Failure Remedies for Physics-Informed Neural Networks},pdfauthor={Jinyuan Zhang, Peng He, Yin Yuan, He Hu, ShengShuo Jiao}}

\theoremstyle{plain}
\newtheorem{proposition}{Proposition}
\newtheorem*{proposition*}{Proposition}
\theoremstyle{remark}
\newtheorem{remark}{Remark}
\newtheorem*{remark*}{Remark}

\graphicspath{{figures/}}

\begin{document}

\twocolumn[{%
  \begin{center}
    {\Large\bfseries Neither Precision Nor Architecture Alone: Controlled Tests of Failure\\
     Remedies for Physics-Informed Neural Networks\par}
    \vspace{1.3em}
    {\normalsize
    \begin{tabular}{c@{\hspace{4em}}c}
      Jinyuan Zhang & Peng He$^{\ast}$\\
      {\ttfamily 202621116012480@stu.hubu.edu.cn} & {\ttfamily penghe@hubu.edu.cn}\\
      \addlinespace[3pt]
      He Hu & Yin Yuan\\
      {\ttfamily 202521120012751@stu.hubu.edu.cn} & {\ttfamily 202521120012766@stu.hubu.edu.cn}\\
      \addlinespace[3pt]
      ShengShuo Jiao & \\
      {\ttfamily 202621120012764@stu.hubu.edu.cn} & \\
    \end{tabular}\par}
    \vspace{1.1em}
    {\itshape Hubei University, Wuhan, China\par}
    \vspace{0.5em}
    {$^{\ast}$Corresponding author: {\ttfamily penghe@hubu.edu.cn}\par}
    \vspace{1.2em}
    \begin{minipage}{0.92\textwidth}
      \begin{abstract}
Physics-Informed Neural Networks (PINNs) frequently fail on stiff or advection-dominated PDEs, and two recent accounts offer competing remedies: switching from FP32 to FP64 to repair an L-BFGS stopping artifact, or replacing the MLP with a state-space-model (SSM) backbone plus sub-sequence alignment to counter architectural simplicity bias. We test both under matched, seed-paired controls in a pre-registered 144-run study spanning convection, reaction, and wave, plus an independent 85-run convection/wave study; success is relative $\ell_2$ error below $0.05$. The two remedies act on disjoint regime-and-seed slices: neither substitutes for the other. On hard convection ($\beta{=}50$), alignment recovers 2/5 seeds in FP32 and 3/5 in FP64, where the unaligned SSM succeeds on 0/5 seeds at either precision and the vanilla MLP moves only from 0/5 to 1/5 across the precision switch---the recoveries trace to the alignment objective, not the backbone. On reaction the backbone alone already succeeds on 3/5--4/5 seeds, so each remedy covers a regime the other does not. Responses are also seed-specific: the same precision switch flips individual seeds in opposite directions and, on wave, lowers median error with no statistically significant success gain. Tightening the inner L-BFGS tolerance in an independent repeated-step runner likewise lowers median error at a large runtime cost, with success counts unchanged. Precision, stopping, backbone, and alignment must therefore be evaluated jointly and reported per seed.
      \end{abstract}
    \end{minipage}
  \end{center}
  \vspace{2.2em}
}]

\section{Introduction}
\label{sec:intro}

Physics-Informed Neural Networks (PINNs)~\cite{raissi2019physics,karniadakis2021physics} embed PDE residuals into the training loss, enabling mesh-free solutions to forward and inverse problems, and have matured into a broad scientific-machine-learning toolkit supported by shared libraries~\cite{lu2021deepxde}. Despite their elegance, PINNs break down on stiff or advection-dominated PDEs---a failure pattern documented across convection, reaction, and wave equations at moderate-to-high parameter regimes~\cite{krishnapriyan2021characterizing}.

Two concurrent 2025 papers offer different diagnoses and cures:

\begin{itemize}
    \item \textbf{PINNMamba}~\cite{pinnmamba2025} attributes failures to \emph{continuous-discrete mismatch} and \emph{simplicity bias} of MLPs, and fixes them with a State Space Model (SSM) backbone plus sub-sequence contrastive alignment.
    \item \textbf{``FP64 is All You Need''}~\cite{fp64paper2025} attributes failures to an \emph{optimizer artifact}---L-BFGS's \texttt{tolerance\_change} sitting near the FP32 machine epsilon ($\sim$1.19e-7)---and fixes them by switching to FP64.
\end{itemize}

These explanations motivate different remedies, but their sufficiency and transferability have not been tested together. If FP64 alone transfers across regimes, it should rescue failures without modifying the model; if the architecture/alignment remedy is sufficient, precision should add little once that protocol is fixed. No existing study tests both interventions within the same controlled slices. Figure~\ref{fig:motivation} separates the two hypotheses and specifies the matched comparison required to test them jointly.

\begin{figure*}[t]
\centering
\includegraphics[width=0.99\textwidth]{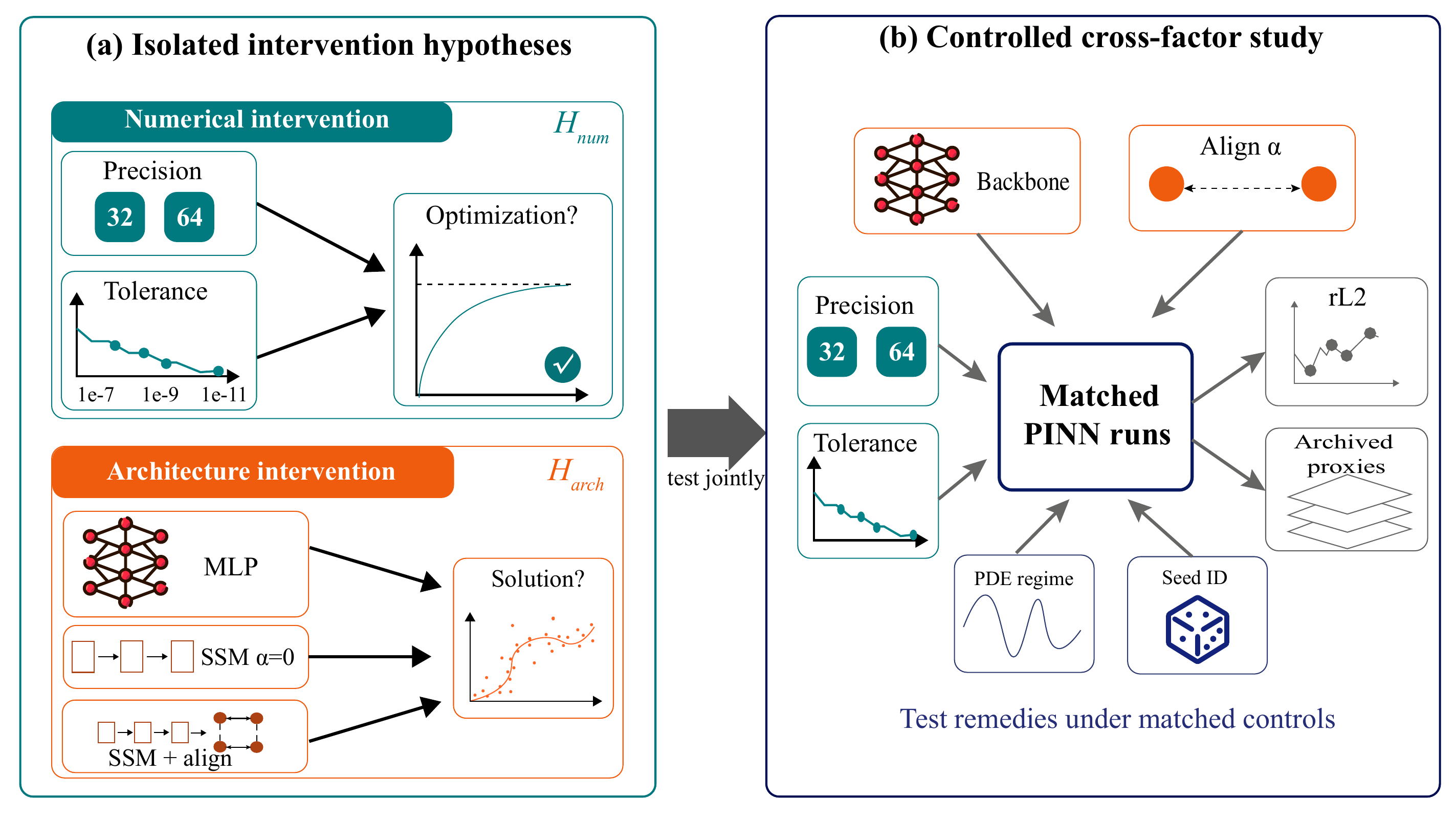}
\caption{Why the two remedies require a joint test. Prior studies isolate numerical precision/stopping from architecture/alignment, whereas our matched comparisons vary precision, inner L-BFGS tolerance, backbone, and alignment weight while controlling the PDE regime and seed. rL2 and archived diagnostic proxies appear as descriptive summaries.}
\label{fig:motivation}
\end{figure*}

The tests come in two complementary studies, each changing one factor at a time. The local controlled study pools 144 runs across preliminary experiments, a pre-registered evaluation, tolerance ablations, protocol comparisons, and matched follow-ups; the independent validation study adds 85 convection and wave runs for the alignment-weight, tolerance, and seed analyses (a wave configuration whose coefficients depart from the public protocol is set aside). Every comparison is matched---a PDE--seed cell shares its collocation grid, budget, and optimizer across variants---and reported per seed, since the remedy that rescues one initialization may leave another failing.

\paragraph{Contributions.}
\begin{enumerate}
    \item A controlled comparison of precision, inner L-BFGS tolerance, backbone, and alignment within matched PDE and seed slices, with the effective sample size reported for each comparison.
    \item Evidence that numerical and architectural interventions act on disjoint regimes and seeds: the alignment-weight sweep exhibits a threshold response, whereas tighter tolerance alters error and runtime but not success counts.
    \item A quantification of seed and threshold sensitivity, together with the scope each conclusion supports.
\end{enumerate}

\section{Related Work}
\label{sec:related}

\paragraph{Diagnosing PINN failure modes.} PINNs fail systematically on stiff or advection-dominated PDEs, and the literature offers several diagnoses that need not exclude each other. \citet{krishnapriyan2021characterizing} documented failures on convection, reaction, and reaction-diffusion equations and attributed them to complex loss landscapes; \citet{wang2021understanding} localized the pathology in the gradient imbalance between residual and boundary terms (cf.\ Eq.~\ref{eq:gradpath}); \citet{wang2022ntk} traced it to the eigenspectrum of the Neural Tangent Kernel; and \citet{wang2022causality} to training schedules that violate temporal causality. Benchmarks and shared libraries have since standardized how such failures are produced and measured~\cite{hao2024pinnacle,lu2021deepxde}. What these accounts share is a design, not only a conclusion: each varies one factor against an otherwise fixed protocol, so which mechanism binds in which regime---and whether a remedy for one transfers to another---has not been tested under matched controls. That joint test is the one our study runs.

\paragraph{Two remedies proposed in 2025.} Two concurrent works propose cures from opposite ends of the stack. On the architecture side, sequential inductive bias has been added through temporal Transformer attention~\cite{pinnsformer2024,vaswani2017attention} and, more recently, through selective state-space (Mamba) blocks~\cite{pinnmamba2025,gu2023mamba}, the latter paired with a sub-sequence contrastive alignment loss and state-of-the-art results on convection-$\beta{=}50$ and reaction-$\rho{=}5$; a sibling line instead swaps the MLP's activation functions to counter spectral bias~\cite{sitzmann2020siren}.
% TODO-cite: other SSM/Mamba applications in scientific computing (see report: Mamba-ND; Zheng et al. 2024 / Hu et al. 2024 as cited within PINNMamba for SSM-based neural operators).
On the optimizer side, \citet{fp64paper2025} attribute the same benchmark failures to L-BFGS's \texttt{tolerance\_change} sitting near the FP32 machine epsilon and report that FP64 arithmetic resolves them. The two papers evaluate overlapping failure cases but under different protocols, report one seed per configuration, and leave their compound remedies unablated---PINNMamba never separates the alignment loss from the SSM backbone, and the FP64 study never tightens the tolerance within FP32. We therefore place both remedies inside shared PDE--seed cells and vary one factor at a time.

\paragraph{Loss weighting and training schedules.} A third family changes neither architecture nor arithmetic but reshapes the objective: gradient-based adaptive balancing~\cite{wang2021understanding}, NTK-based weights~\cite{wang2022ntk}, co-trained self-adaptive masks~\cite{mcclenny2023selfadaptive}, and causal temporal weighting~\cite{wang2022causality}, alongside curriculum schedules that march the training window forward in time~\cite{krishnapriyan2021characterizing}.
% TODO-cite: loss re-weighting / adaptive weighting beyond the above (see report: van der Meer et al. 2022; Xiang et al. 2022; Bischof & Kraus 2021; Maddu et al. 2022).
% TODO-cite: curriculum & causal training (see report: Wight & Zhao 2020; Daw et al. 2023).
Read this way, sub-sequence alignment belongs to the same family: it adds an agreement term to the objective rather than a new optimizer or a new number format. Our protocol holds the rest of this family fixed and sweeps the alignment term by weight, so its contribution separates from the backbone that carries it.

\paragraph{The optimizer and precision view.} L-BFGS~\cite{liu1989lbfgs} is the de facto second-order optimizer in PINN training, and its convergence tests cannot resolve progress below the roundoff of the working arithmetic~\cite{higham2002accuracy}. Loss-landscape analysis pushes toward the same conclusion from the other side: PINN objectives are ill-conditioned enough to defeat first-order methods, which is why second-order optimizers dominate practice~\cite{rathore2024challenges}.
% TODO-cite: optimizer-focused studies of L-BFGS in PINN training (see report: Markidis 2021; Taylor et al. 2025; Basir & Senocak 2022).
The FP64 remedy rests on exactly this foundation: if the stopping floor is set by arithmetic, raising precision lowers the floor. What the foundation does not predict is whether the lowered floor turns failures into successes, in which regimes, and for which seeds; that part is empirical, and we answer it per seed.

\paragraph{Study design.} Prospective registration is standard in clinical trials but rare in machine learning, and reproducibility initiatives~\cite{pineau2021improving} motivate our frozen evaluation rules. Concurrent work maps optimizer effectiveness across regimes spanning PINNs, neural operators, and neural ODEs~\cite{wang2026multiregime}. Our contribution is narrower and complementary: precision, stopping, backbone, and alignment held under matched PINN controls, compared seed by seed, with the evaluation rules registered before the study expanded.

\section{Method}
\label{sec:method}

Figure~\ref{fig:framework} summarizes how a matched run is constructed, trained, and analyzed.

\begin{figure*}[t]
\centering
\includegraphics[width=0.86\textwidth]{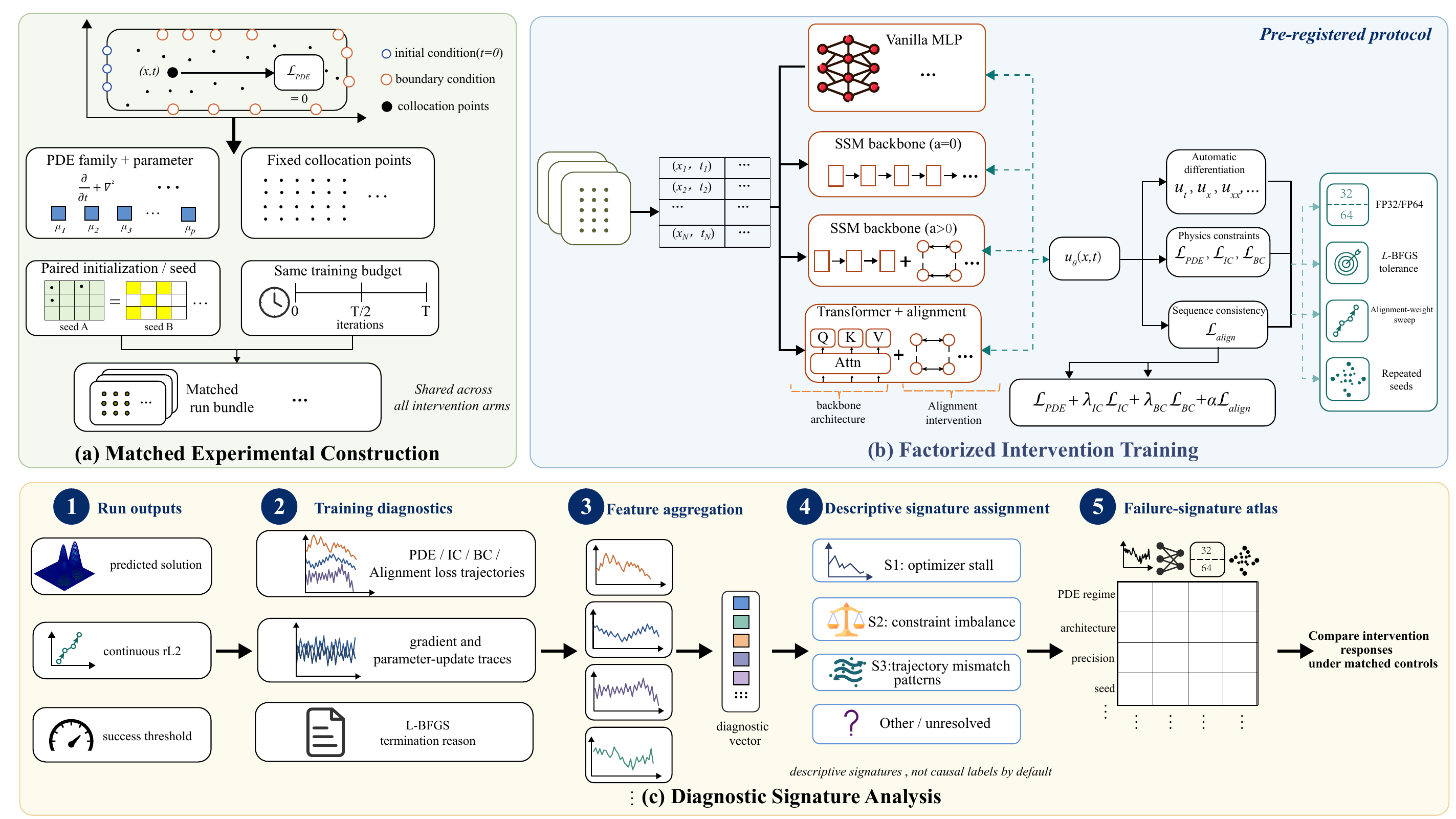}
\caption{Overview of the controlled study workflow.}
\label{fig:framework}
\end{figure*}

\subsection{Problem Setup}
\label{sec:setup}

We study three one-dimensional PDEs that are standard PINN stress tests~\cite{krishnapriyan2021characterizing}: the convection equation $u_t + \beta u_x = 0$ at $\beta{=}50$, whose transported profile develops steep, high-frequency gradients; the stiff reaction equation $u_t = \rho u(1-u)$ at $\rho{=}5$, with a rapidly saturating solution; and the wave equation $u_{tt}-4u_{xx}=0$ under a two-frequency initial displacement, whose multi-scale oscillations stress temporal coherence. All three are posed on a periodic spatial domain and trained against analytic reference solutions on a fine grid; the parameters are deliberately hard---each regime fails under a standard MLP trained with L-BFGS, the joint baseline both remedies are meant to rescue.

Both remedies target the same objective, so it is worth stating explicitly. Writing each PDE in operator form $M[u]=0$ on $\Omega\times[0,T]$ with boundary operator $B$ and initial operator $I$, the network $u_\theta$ minimizes the physics-informed objective
\begin{align}
\mathcal{L}(\theta) &= w_r\mathcal{L}_r(\theta)+w_b\mathcal{L}_b(\theta)+w_i\mathcal{L}_i(\theta), \label{eq:loss}\\
\mathcal{L}_r(\theta) &= \tfrac{1}{|\chi_r|}\textstyle\sum_{(x,t)\in\chi_r}\|M[u_\theta](x,t)\|^2, \nonumber
\end{align}
with $\mathcal{L}_b,\mathcal{L}_i$ the mean-squared boundary and initial residuals over their collocation sets. The three regimes are precisely where this objective is hard to minimize: the residual gradient dominates the data gradients in the sense of the pathology ratio~\cite{wang2021understanding},
\begin{equation}\label{eq:gradpath}
\frac{\|\nabla_\theta\mathcal{L}_r\|}{\|\nabla_\theta\mathcal{L}_b\|+\|\nabla_\theta\mathcal{L}_i\|}\gg 1,
\end{equation}
equivalently the Neural Tangent Kernel is ill-conditioned, $\lambda_{\max}(K_r)/\lambda_{\min}(K_{b,i})\gg1$~\cite{wang2022ntk}. Each remedy attempts to restore minimizability of~\eqref{eq:loss}, but along a different axis.

\subsection{Remedy Mechanisms}
\label{sec:mechanism}

The two propositions below make precise why these interventions are not interchangeable.

\textbf{Precision acts on the optimizer's stopping test.} L-BFGS terminates when the iterate change drops below a \texttt{tolerance\_change} threshold $\tau$~\cite{liu1989lbfgs}, and that threshold cannot be resolved below the roundoff at which the objective is evaluated.
\begin{proposition}[Precision-bounded stopping floor]\label{prop:floor}
Let $\mathcal{L}$ be evaluated in precision $p$ with unit roundoff $u_p$ (FP32: $u\approx5.96\!\times\!10^{-8}$, machine $\varepsilon\approx1.19\!\times\!10^{-7}$; FP64: $u\approx1.11\!\times\!10^{-16}$). The iterate difference $|\Delta\mathcal{L}_k|$ driving the stopping test satisfies the roundoff floor $|\Delta\mathcal{L}_k|\gtrsim u_p|\mathcal{L}_k|$, so the test resolves genuine progress only down to $\max(\tau,\,u_p|\mathcal{L}_k|)$.
\end{proposition}
\noindent\emph{Intuition.} The optimizer cannot resolve progress below the loss rounding error: in FP32 this floor coincides with the default $\tau$, while in FP64 it lies nine orders of magnitude below it. The full proof (with the IEEE-754 error expansion) is in Appendix~A.1 of the supplement.

\textbf{Alignment acts on the hypothesis class.} The second remedy adds the sub-sequence contrastive term of \citet{pinnmamba2025},
\begin{align}
\mathcal{L}_{\mathrm{align}}(\theta) &= \tfrac{1}{(k\!-\!1)|\chi|}\textstyle\sum_{(x,t)\in\chi}\sum_{j=1}^{k-1}\Delta_j(x,t;\theta), \label{eq:align}\\
\Delta_j &= \big\|u_\theta^{(s_0)}(x,t{+}j\Delta t)-u_\theta^{(s_j)}(x,t{+}j\Delta t)\big\|^2, \nonumber
\end{align}
where $u_\theta^{(s)}$ is the prediction issued by sub-sequence $s$, so the trained objective is $\mathcal{L}+\alpha\mathcal{L}_{\mathrm{align}}$. It is one instance of the temporal-conditioning family. Causal residual weighting~\cite{wang2022causality} replaces $\mathcal{L}_r$ by a causally weighted form,
\begin{gather}
\mathcal{L}_c = \tfrac{1}{|\chi|}\textstyle\sum_{(x,t)\in\chi} w_c(x,t)\|M[u_\theta](x,t)\|^2,\\
w_c(x,t) \propto \exp\!\big(\textstyle-\varepsilon\!\sum_{t'\le t}\|M[u_\theta](x,t')\|^2\big), \nonumber
\end{gather}
which suppresses residual error at time $t$ until earlier times are fit; self-adaptive masks~\cite{mcclenny2023selfadaptive} instead use $\mathcal{L}=\sum_i s_i\|M[u_\theta](x_i,t_i)\|^2$ with co-trained weights $s_i\ge0$. All three modify the hypothesis class, not the stopping test.

\begin{proposition}[Alignment as over-determination]\label{prop:align}
Let $\mathcal{F}_{\chi^*}=\{u_\theta:M[u_\theta]|_{\chi^*}=0\}$ be the collocation-feasible set, which by Theorem~4.1 of \citet{pinnmamba2025} contains infinitely many spurious solutions. The agreement term~\eqref{eq:align} imposes $m=(k{-}1)|\chi|$ equality constraints $u_\theta^{(s_0)}=u_\theta^{(s_j)}$ at shared collocation points; whenever these are independent on the hypothesis class, the feasible set shrinks to $\mathcal{F}_{\chi^*}\cap\{\text{agreements}\}$ of dimension $\dim\mathcal{F}_{\chi^*}-m$, removing the bump-type spurious solutions.
\end{proposition}
\noindent\emph{Intuition.} Each cross-subsequence agreement is an independent level-set that lowers the feasible manifold's dimension by one; a bump-type spurious solution violates at least one such agreement and is removed (Theorem~4.1 of \citet{pinnmamba2025}). The full proof (bump-function construction) is in Appendix~A.2 of the supplement.

\noindent Propositions~\ref{prop:floor}--\ref{prop:align} formalize our central claim: precision acts on the optimizer's stopping axis and alignment on the hypothesis-class axis, so the two remedies lie on orthogonal axes of the loss landscape~\cite{rathore2024challenges} and neither subsumes the other.

\subsection{Experimental Protocol}
\label{sec:protocol}

The local controlled study pools 144 runs from preliminary experiments, the registered evaluation, tolerance ablations, protocol comparisons, and matched follow-ups. The independent validation study contributes 85 runs and 55.7 run-hours. Because neither study forms a complete factorial block, Table~\ref{tab:coverage} reports the effective sample size for each research question (RQ) instead of merging overlapping slices into a single total.

\begin{table}[t]
\centering
\caption{Evidence coverage. Local controlled-study slices overlap; independent-study slices partition the 85 runs, with one non-comparable slice excluded from validation claims. Asterisked proxy counts have unequal observability.}
\label{tab:coverage}
\begin{tabular}{p{0.60\columnwidth}cc}
\toprule
\textbf{Slice} & \textbf{RQ} & $\boldsymbol{n}$ \\
\midrule
Local study: five-seed conv/react slice & 1--2,4 & 60 \\
Local study: instrumented conv/react subset & 5 & $87^*$ \\
Local study: model-protocol comparison & 2 & 36 \\
Independent study: convection $\alpha$ sweep & 2 & 25 \\
Independent study: convection tolerance sweep & 3 & 30 \\
Independent study: wave MLP precision slice & 4 & 20 \\
Independent study: excluded wave-PINNMamba slice & -- & 10 \\
\bottomrule
\end{tabular}
\par\vspace{2pt}\raggedright\footnotesize $^*$S1 is observable for 46/87 records (39/63 failures); S3b plateau data for 18/87 (13/63). The registered S2 peak-amplitude condition is absent from the archive.
\end{table}

Within each comparison we control PDE (convection $\beta{=}50$, reaction $\rho{=}5$, wave with a two-frequency initial displacement), architecture (vanilla MLP; SSM without alignment, $\alpha{=}0$; SSM with alignment; PINNsFormer in a smaller slice), precision (FP32, FP64), seeds (five per main cell; three or five-to-ten in smaller/independent slices), and stopping/alignment (\texttt{tolerance\_change}$\in\{10^{-7},10^{-9},10^{-11}\}$ on the convection MLP sweep; $\alpha\in\{0,100,300,1000,3000\}$ on the SSM sweep). Runs in a comparison share the same PDE instance, collocation grid, training budget, and L-BFGS strong-Wolfe configuration. The registered evaluation uses 5{,}000 Adam steps then up to 1{,}000 outer L-BFGS steps on a $101\times101$ collocation grid with a $501\times501$ evaluation grid and disabled TF32; the independent study uses a repeated-step runner (1{,}000 L-BFGS calls, \texttt{max\_iter}$=20$) on an NVIDIA RTX 4090D. Accuracy is the relative $\ell_2$ error $\mathrm{rL2}=\|u_\theta-u_{\mathrm{ref}}\|_2/\|u_{\mathrm{ref}}\|_2$ over the evaluation grid, with success at unrounded rL2 $<0.05$; each cell reuses the same seeds across its variants so differences reflect the intervention.

\subsection{Pre-registration and Provenance}
\label{sec:prereg}

To separate confirmatory from exploratory evidence, we registered the analysis plan on May 28, 2026 after a set of preliminary experiments and before expanding the study. The registered plan fixed the success threshold (rL2 $<0.05$), the diagnostic signature definitions (S1 optimizer stall; S2 constraint-mismatch collapse; S3a--c trajectory conditions), and the rule that demotes the wave PDE when too few cells admit a clean signature. The supplementary archive contains the registered document, its timestamp, and a provenance table marking each run as preliminary, registered, follow-up, or independent; the independent validation study and the local follow-ups were run after registration.

\subsection{Archived Diagnostic Proxies}
\label{sec:proxies}

The registered classifier tags each failure as optimizer stall (S1), constraint mismatch (S2), or a trajectory condition (S3a--c). The archived implementation covers these rules in part (two registered tests absent; 41/87 records without L-BFGS logs), so the proxies enter as descriptive summaries, with unavailable evidence read as ``not fired''; the complete rules and coverage are deferred to the supplement.

\section{Experiments}
\label{sec:experiments}

Five research questions organize the experiments:
\begin{itemize}
    \item \textbf{RQ1:} Are precision and architecture/alignment interchangeable fixes under matched controls?
    \item \textbf{RQ2:} Which part of the architecture-based remedy matters: alignment weight or the recorded model protocol?
    \item \textbf{RQ3:} How do precision and inner L-BFGS tolerance affect the repeated-step runner?
    \item \textbf{RQ4:} How stable are intervention responses across seeds and evaluation choices?
    \item \textbf{RQ5:} Which archived diagnostic proxies appear in the instrumented convection/reaction subset?
\end{itemize}

\begin{figure*}[t]
\centering
\includegraphics[width=0.86\textwidth]{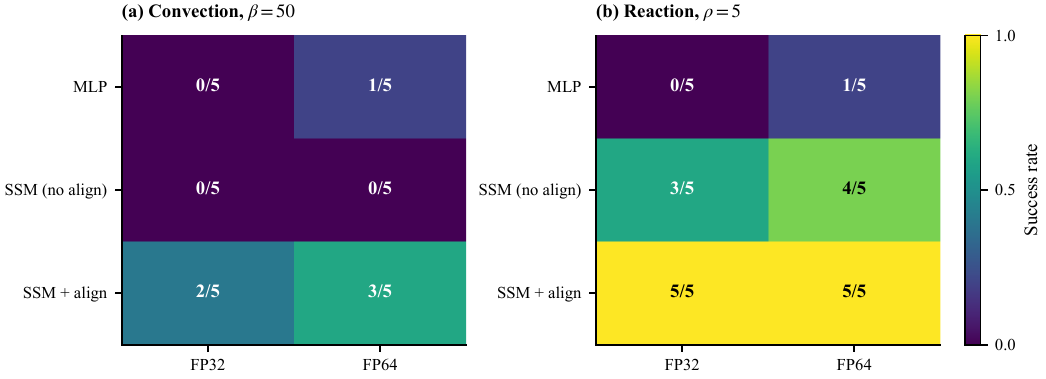}
\caption{Local controlled-study five-seed slice (success defined as rL2 $<0.05$), pooling registered and matched follow-up runs; FP64 no-alignment seeds 3--4 are drawn from the follow-ups. Precision and architecture/alignment produce regime- and seed-specific gains, with no uniform monotone improvement.}
\label{fig:heatmap}
\end{figure*}

\subsection{Cross-Factor Comparison (RQ1)}

Precision and alignment act on different PDE and seed slices, and neither substitutes for the other. Figure~\ref{fig:heatmap} shows the local controlled-study response patterns, and Table~\ref{tab:atlas} reports per-seed rL2 for the main convection/reaction slice at the registered outcome threshold. On hard convection, FP64 lifts vanilla-MLP success from 0/5 to 1/5 and leaves the unaligned SSM at 0/5, whereas alignment recovers 2/5 seeds in FP32 and 3/5 in FP64. On reaction, the SSM backbone already reaches 3/5--4/5 and alignment reaches 5/5 at both precisions, so the two remedies divide the regimes between them. The wave, transformer, tolerance, and alignment-weight follow-ups answer the remaining RQs.

\begin{table*}[t]
\centering
\caption{Per-seed rL2 in the local controlled-study slice (five seeds per cell; bold is rL2 $<0.05$ on unrounded values).}
\label{tab:atlas}
\begin{tabular}{llcccccc}
\toprule
\textbf{PDE} & \textbf{Regime} & \textbf{s=0} & \textbf{s=1} & \textbf{s=2} & \textbf{s=3} & \textbf{s=4} & \textbf{Rate} \\
\midrule
\multirow{6}{*}{Conv-$\beta$50}
& mlp FP32 & 1.005 & 0.896 & 0.806 & 0.914 & 0.800 & 0/5 \\
& mlp FP64 & 0.080 & 0.738 & 0.077 & \textbf{0.045} & 0.052 & 1/5 \\
& ssm\_noalign FP32 & 2.027 & 1.361 & 1.188 & 1.327 & 1.372 & 0/5 \\
& ssm\_noalign FP64 & 1.312 & 1.400 & 1.298 & 1.238 & 1.326 & 0/5 \\
& ssm\_align FP32 & \textbf{0.011} & \textbf{0.0498} & 0.549 & 0.386 & 0.904 & 2/5 \\
& ssm\_align FP64 & \textbf{0.006} & 1.287 & \textbf{0.005} & 0.168 & \textbf{0.026} & 3/5 \\
\midrule
\multirow{6}{*}{React-$\rho$5}
& mlp FP32 & 0.979 & 0.979 & 0.079 & 0.980 & 0.974 & 0/5 \\
& mlp FP64 & 0.055 & 0.065 & 0.059 & 0.064 & \textbf{0.042} & 1/5 \\
& ssm\_noalign FP32 & \textbf{0.026} & 0.106 & \textbf{0.030} & \textbf{0.031} & 0.682 & 3/5 \\
& ssm\_noalign FP64 & \textbf{0.034} & 0.643 & \textbf{0.035} & \textbf{0.029} & \textbf{0.023} & 4/5 \\
& ssm\_align FP32 & \textbf{0.019} & \textbf{0.028} & \textbf{0.010} & \textbf{0.022} & \textbf{0.031} & 5/5 \\
& ssm\_align FP64 & \textbf{0.028} & \textbf{0.031} & \textbf{0.026} & \textbf{0.022} & \textbf{0.023} & 5/5 \\
\bottomrule
\end{tabular}
\end{table*}

\subsection{Alignment Contribution and Weight Response (RQ2)}

The hard-convection SSM recoveries require alignment, and the response depends on both the alignment weight and the PDE. Table~\ref{tab:noalign} isolates the backbone's contribution on convection-$\beta{=}50$:

\begin{center}
\begin{minipage}{0.85\columnwidth}
\centering
\begin{tabular}{lcc}
\toprule
& FP32 & FP64 \\
\midrule
ssm\_noalign & 0/5 (all $>$1.0) & 0/5 (all $>$1.0) \\
ssm\_align & 2/5 & 3/5 \\
\bottomrule
\end{tabular}
\captionof{table}{Hard-convection success by alignment and precision (five seeds per cell).}
\label{tab:noalign}
\end{minipage}
\end{center}

Without alignment, the SSM backbone yields rL2 above 1.0 in every reported cell, with no improvement over the vanilla MLP. Under this protocol, the SSM-based recoveries are therefore attributable to the alignment objective rather than to the backbone alone. On reaction-$\rho{=}5$, the backbone alone achieves 3/5 (FP32) and 4/5 (FP64), while alignment lifts both precisions to 5/5; alignment thus has a larger effect on convection than on reaction.

\paragraph{Weight response.} The independent sweep reuses the same five seeds at $\alpha\in\{0,100,300,1000,3000\}$ on convection PINNMamba FP32. Success counts are 0/5, 0/5, 0/5, 2/5, and 3/5, and median rL2 falls from 0.972 at $\alpha{=}300$ to 0.801 at $\alpha{=}1000$ and 0.041 at $\alpha{=}3000$. Figure~\ref{fig:alpha-response} shows a threshold-like response rather than a sharp phase transition; at five seeds on a coarse grid, the trend is clear but the threshold is not precisely calibrated.

\paragraph{Recorded model protocols.} The local controlled study also compares PINNMamba with PINNsFormer. Table~\ref{tab:protocol} reports the convection-$\beta{=}50$ success counts.

\begin{center}
\begin{minipage}{0.85\columnwidth}
\centering
\begin{tabular}{lcc}
\toprule
Convection-$\beta{=}50$ & FP32 & FP64 \\
\midrule
PINNMamba & 2/5 & 3/5 \\
PINNsFormer & 0/5 & 0/3 \\
\bottomrule
\end{tabular}
\captionof{table}{PINNMamba vs.\ PINNsFormer success on convection-$\beta{=}50$ (unequal seeds; protocol-level contrast).}
\label{tab:protocol}
\end{minipage}
\end{center}

PINNsFormer records no successful convection seed (all rL2 $>0.97$); on reaction, both protocols solve every cell. The PINNsFormer manifest leaves the alignment field blank and the methods differ beyond backbone class, so the contrast here is protocol-level; representative trajectories are in the supplement.

\begin{center}
\begin{minipage}{0.95\columnwidth}
\centering
\includegraphics[width=\linewidth]{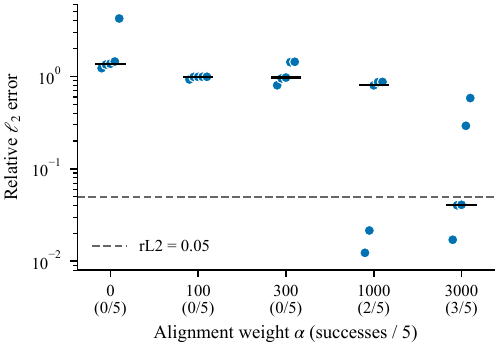}
\captionof{figure}{Independent convection PINNMamba FP32 sweep (the same five seeds per $\alpha$). Points are seeds, black segments are medians, and tick parentheses give successes/5. The dashed line is rL2 $=0.05$; points below it are successes.}
\label{fig:alpha-response}
\end{minipage}
\end{center}

\subsection{Tolerance Response in the Repeated-Step Runner (RQ3)}

Tighter inner tolerance lowers error and raises cost without changing success counts in the repeated-step implementation. The independent validation study issues 1{,}000 optimizer calls, each capped at 20 inner L-BFGS iterations, and sweeps \texttt{tolerance\_change} across three orders of magnitude on convection MLP. Because this protocol differs from a single-call L-BFGS run, it probes the tolerance response of a repeated-step solver rather than reproducing the single-call premature-termination mechanism. Table~\ref{tab:tolerance} reports success count and median rL2 per cell.

\begin{center}
\begin{minipage}{0.95\columnwidth}
\centering
\begin{tabular}{lccc}
\toprule
Precision & tol=1e-7 & tol=1e-9 & tol=1e-11 \\
\midrule
FP32 & 1/5; 1.006 & 1/5; 0.720 & 1/5; 0.720 \\
FP64 & 1/5; 1.005 & 1/5; 0.314 & 1/5; 0.314 \\
\bottomrule
\end{tabular}
\captionof{table}{Repeated-step convection MLP sweep: success count and median rL2 by precision and \texttt{tolerance\_change}.}
\label{tab:tolerance}
\end{minipage}
\end{center}

Tightening tolerance leaves the success count at 1/5 for both precisions while lowering the median error in each, more markedly in FP64. Median runtime rises from 120.8 to 2{,}573.3--2{,}569.6 seconds in FP64 and from 17.0 to 278.1--275.7 seconds in FP32, so the error gain---which concentrates at the step from $10^{-7}$ to $10^{-9}$ and then saturates---comes at a sharply negative marginal return per second. Figure~\ref{fig:tolerance-response} shows the seed distributions. Inner-iteration counts and termination reasons are absent from the logs, so these numbers describe a precision-dependent tolerance response of the repeated-step runner, not the optimizer's internal interaction or mechanism.

\begin{center}
\begin{minipage}{0.95\columnwidth}
\centering
\includegraphics[width=\linewidth]{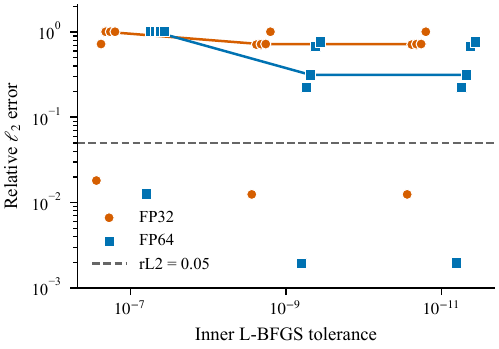}
\captionof{figure}{Independent repeated-step convection MLP sweep (the same five seeds per precision/tolerance cell). Circles are FP32, squares are FP64, and connected markers are medians. The dashed line is rL2 $=0.05$; points below it are successes.}
\label{fig:tolerance-response}
\end{minipage}
\end{center}

\subsection{Seed and Evaluation Robustness (RQ4)}

Intervention responses vary sharply across seeds. In the local controlled-study convection SSM+alignment slice, FP64 moves seed~1 from 0.0498 to 1.287 but seed~2 from 0.549 to 0.005---opposite directions under the same intervention. Reaction SSM+alignment succeeds in all ten precision--seed cells, whereas the no-alignment reaction arm still fails one FP64 cell. The independent ten-seed wave-MLP slice yields 4/10 successes in FP32 (median rL2 0.125) and 6/10 in FP64 (median 0.045). Among the paired seeds, three succeed only in FP32, five only in FP64, one in both, and one in neither; exact McNemar $p{=}0.727$, and a paired Wilcoxon test on log-rL2 gives $p{=}0.375$. FP64 therefore lowers the median error, though the success-rate gain is not statistically significant. The original three-seed wave-architecture result is exploratory, and the non-comparable independent wave-PINNMamba runs sit outside this analysis.

Table~\ref{tab:threshold} sweeps the success threshold $\pm$40\% (0.03--0.07) on the 60-run local convection/reaction slice.

\begin{center}
\begin{minipage}{0.95\columnwidth}
\centering
\begin{tabular}{cccccc}
\toprule
Threshold & 0.03 & 0.04 & \textbf{0.05} & 0.06 & 0.07 \\
\midrule
Success rate & 25\% & 35\% & \textbf{40\%} & 45\% & 48\% \\
\bottomrule
\end{tabular}
\captionof{table}{Success-rate sensitivity to the rL2 threshold on the 60-run local convection/reaction slice.}
\label{tab:threshold}
\end{minipage}
\end{center}

The aggregate regime ordering is stable across this range, although individual near-threshold seeds change status; the smooth rise in the success rate as the cutoff loosens confirms that the rankings are not artifacts of the 0.05 boundary.

\subsection{Diagnostic Failure Signatures (RQ5)}

Among the 63 failures in the 59-convection/28-reaction subset, the dominant proxy class is trajectory-related, consistent with the late-time and multi-scale distortions seen on convection and wave; the unresolved share reflects the missing L-BFGS logs noted in Archived Diagnostic Proxies. The exclusive counts (stall / constraint / trajectory / mixed / unresolved) and the bar chart are in the supplement.

\FloatBarrier

\section{Discussion}
\label{sec:discussion}

\paragraph{Revisiting the two diagnoses.} Both original diagnoses survive matched controls, each within a narrower scope than first claimed. The precision diagnosis is right that arithmetic bounds the L-BFGS stopping test---FP64 lowers error in every controlled cell---yet precision alone rarely flips a hard-convection failure into a success (1/5), which makes it a complement to, rather than a substitute for, the architectural remedy. The architecture diagnosis is right that the SSM with alignment recovers hard convection, but our ablation credits the alignment objective rather than the Mamba backbone, since the unaligned SSM matches or underperforms the vanilla MLP. Each remedy holds where it was reported, and each covers a different slice of the regimes and seeds we test.

\paragraph{Why seed-level reporting matters.} On hard convection with alignment, FP64 turns seed~2 from failure into success but seed~1 from success into failure, so the cell-level rate (3/5) hides opposite per-seed movements; reversals like these are characterized through their distribution, not confirmed by another seed---the case for per-seed reporting.

\paragraph{Implications.} Neither ``just use FP64'' nor ``just change the backbone'' qualifies as general advice: practitioners should compare precision and stopping jointly, tune alignment as an intervention, and report per-seed error distributions (for lower error, FP64 at moderate tolerance is the efficient point, since tightening past $10^{-9}$ multiplies runtime without further gain). Single-configuration conclusions shift when seed, precision, stopping, and alignment vary jointly; causal interpretation of the proxies awaits counterfactual continuations.

\paragraph{Limitations.} The corpus covers three one-dimensional PDE families at selected parameter values, not a broad scientific-computing benchmark. The PINNMamba--PINNsFormer comparison is a method-protocol comparison with unequal seeds, a blank PINNsFormer alignment field, no parameter/tuning match, and no identical initialization tensors across architectures. The independent validation uses a repeated-step runner, and its wave-PINNMamba coefficient differs from the public wave protocol. The archived proxy classifier omits two registered conditions, 41/87 records lack L-BFGS logs, S3 uses final rL2, and the causal taxonomy still awaits validation by rescue experiments.

\section{Reproducibility}

All random seeds are reported, and rL2 $<0.05$ is applied to unrounded values. Artifacts include the run manifest, registered analysis file, archived proxy JSON, independent-validation CSV and implementation, and the scripts that regenerate Figures~\ref{fig:heatmap}--\ref{fig:tolerance-response}; provenance separates preliminary, registered, follow-up, and independent evidence, and the pinned upstream code and environment accompany the submission.

\section{Conclusion}

Under matched controls, precision and tolerance act on different PDE and seed slices than architecture and alignment, so the two classes of remedy are not interchangeable. The alignment recoveries, tolerance costs, and seed reversals we observe make a practical case: evaluate precision, stopping, backbone, and alignment jointly, and report every outcome per seed.

\bibliographystyle{plainnat}
\bibliography{refs}

% ============================ Appendix ================================
% Supplementary material of the conference version (full proofs),
% merged for arXiv.
\onecolumn
This supplement provides the full proofs of the two propositions stated in Section~3 (Remedy Mechanisms) of the main paper. We write $u_p$ for the unit roundoff of precision $p$ (FP32: $u\approx 5.96\times10^{-8}$, machine $\varepsilon\approx1.19\times10^{-7}$; FP64: $u\approx1.11\times10^{-16}$), $\mathrm{fl}(\cdot)$ for floating-point evaluation, $\tau$ for the L-BFGS \texttt{tolerance\_change} threshold, and $\mathcal{L}$ for the physics-informed objective of Eq.~(1) of the main text.

\section*{Appendix A. Proofs}

\subsection*{A.1 Proof of Proposition 1 (precision-bounded stopping floor)}

We restate the proposition and prove it from the IEEE-754 floating-point model.

\begin{proposition*}[Precision-bounded stopping floor]
Let $\mathcal{L}$ be evaluated in precision $p$ with unit roundoff $u_p$. The iterate difference $|\Delta\mathcal{L}_k|=|\mathcal{L}_{k+1}-\mathcal{L}_k|$ that drives the L-BFGS \texttt{tolerance\_change} test satisfies the roundoff floor $|\mathrm{fl}(\Delta\mathcal{L}_k)|\gtrsim u_p|\mathcal{L}_k|$, so the test resolves genuine progress only down to $\max(\tau,\,u_p|\mathcal{L}_k|)$.
\end{proposition*}

\begin{proof}
By the IEEE-754 floating-point model (Higham 2002), every computed scalar satisfies $\mathrm{fl}(a)=a(1+\delta_a)$ with $|\delta_a|\le u_p$, equivalently $|\mathrm{fl}(a)-a|\le u_p|a|$. Writing $\mathrm{fl}(\mathcal{L}_k)=\mathcal{L}_k+e_k$ with $|e_k|\le u_p|\mathcal{L}_k|$, the computed iterate difference decomposes as
\begin{equation}
\mathrm{fl}(\mathcal{L}_{k+1})-\mathrm{fl}(\mathcal{L}_k)=(\mathcal{L}_{k+1}-\mathcal{L}_k)+(e_{k+1}-e_k),
\label{eq:roundoff-decomp}
\end{equation}
and the second term is bounded in magnitude by $u_p(|\mathcal{L}_{k+1}|+|\mathcal{L}_k|)\approx 2u_p|\mathcal{L}_k|$ near convergence, where $|\mathcal{L}_{k+1}|\approx|\mathcal{L}_k|$. This term is an irreducible noise floor: any genuine decrease $(\mathcal{L}_{k+1}-\mathcal{L}_k)$ smaller than $2u_p|\mathcal{L}_k|$ is masked by rounding error and indistinguishable from zero. The L-BFGS stopping criterion $|\mathrm{fl}(\mathcal{L}_{k+1})-\mathrm{fl}(\mathcal{L}_k)|<\tau$ (Liu \& Nocedal 1989) therefore resolves genuine progress only down to $\max(\tau,\,2u_p|\mathcal{L}_k|)$, i.e.\ $\max(\tau,\,u_p|\mathcal{L}_k|)$ up to the harmless factor of two.

Substituting the roundoffs settles the precision comparison. With the default \texttt{tolerance\_change} $\tau\approx1.19\times10^{-7}$ and the loss magnitudes $|\mathcal{L}_k|=O(1)$--$O(10^{2})$ typical of the residual-dominated PINN regimes we study, the FP32 floor $2u_{\mathrm{FP32}}|\mathcal{L}_k|\sim10^{-7}$--$10^{-5}$ is of the same order as or exceeds $\tau$, so the test sits at the noise floor and may terminate while reducible progress remains. In FP64 the floor drops to $2u_{\mathrm{FP64}}|\mathcal{L}_k|\sim10^{-15}$--$10^{-13}$, nine orders of magnitude below $\tau$, restoring the test's ability to detect progress of size $\tau$. Hence FP64 relaxes the stopping floor that FP32 imposes, without changing the model.
\end{proof}

\subsection*{A.2 Proof of Proposition 2 (alignment as over-determination)}

We build on Theorem~4.1 of Xu et al.\ (2025), which shows that for any finite collocation set $\chi^*\subset\Omega\times[0,T]$ there exist infinitely many functions $u_\theta$ (constructed as bump perturbations supported in small neighborhoods of $\chi^*$) with $M[u_\theta]|_{\chi^*}=0$ yet $M[u_\theta]\ne0$ almost everywhere on $\Omega\times[0,T]\setminus\chi^*$.

\begin{proposition*}[Alignment as over-determination]
Let $\mathcal{F}_{\chi^*}=\{u_\theta:M[u_\theta]|_{\chi^*}=0\}$ be the collocation-feasible set. The sub-sequence agreement term of Eq.~(3) in the main text imposes $m=(k{-}1)|\chi|$ equality constraints; whenever these are independent on the hypothesis class, the feasible set shrinks to $\mathcal{F}_{\chi^*}\cap\{\text{agreements}\}$ of dimension $\dim\mathcal{F}_{\chi^*}-m$, removing the bump-type spurious solutions of Xu et al.\ (2025).
\end{proposition*}

\begin{proof}
We prove the two claims in turn.

\emph{(a) Dimension reduction.} The agreement term penalizes, at each shared collocation point $(x,t+j\Delta t)$, the disagreement between the predictions issued by the sub-sequences $s_0,\dots,s_j$ that all contain it. On its zero set this is the system of $m=(k{-}1)|\chi|$ equality constraints
\begin{equation}
u_\theta^{(s_0)}(x,t+j\Delta t)=u_\theta^{(s_j)}(x,t+j\Delta t),\qquad j=1,\dots,k-1,\ (x,t)\in\chi .
\label{eq:agreement-system}
\end{equation}
Assume these constraints are functionally independent on the hypothesis class---a generic condition, since each constrains a distinct pair of sub-sequence outputs at a distinct point. By the regular level-set theorem (a corollary of the implicit function theorem), each independent equality constraint lowers the dimension of the feasible manifold $\mathcal{F}_{\chi^*}$ by one. Hence the augmented feasible set $\mathcal{F}_{\chi^*}\cap\{\text{agreements}\}$ has dimension $\dim\mathcal{F}_{\chi^*}-m$.

\emph{(b) Removal of bump-type spurious solutions.} Let $\bar u$ denote the initial-condition-propagating reference solution. A bump-type spurious $u_\theta$ from Theorem~4.1 of Xu et al.\ (2025) agrees with $\bar u$ on $\chi^*$ (so $M[u_\theta]|_{\chi^*}=0$) but departs from $\bar u$ on $\Omega\times[0,T]\setminus\chi^*$. Now consider how a sub-sequence $s_j$, which is trained to predict the solution forward from time $t+j\Delta t$, renders a shared point $(x,t+j\Delta t)$. The propagating solution $\bar u$ is a fixed point of this forward prediction, so it is rendered identically by every sub-sequence that contains the point: $\bar u^{(s_j)}=\bar u^{(s_0)}$. A spurious $u_\theta$ that fails to propagate the initial condition instead drifts across sub-sequences, so $u_\theta^{(s_j)}\ne u_\theta^{(s_0)}$ at some shared point, giving $\Delta_j>0$ for at least one $j$. Such a $u_\theta$ therefore violates the agreement set and is excluded from $\mathcal{F}_{\chi^*}\cap\{\text{agreements}\}$. With $m$ large relative to the local degrees of freedom deposited near $\chi^*$ by the bump construction, all bump-type spurious solutions are excluded, leaving $\bar u$ as the dominant feasible hypothesis.
\end{proof}

\begin{remark}
The independence assumption in part~(a) is the only non-trivial hypothesis; it holds generically because the agreement constraints couple disjoint sub-sequence outputs at disjoint points. Part~(b) formalizes the design rationale of the sub-sequence contrastive loss: alignment enforces the cross-subsequence agreement that the propagating solution satisfies but collocation-only spurious solutions do not.
\end{remark}

\end{document}